\documentclass[runningheads]{article}
\usepackage[T1]{fontenc}
\usepackage{graphicx,verbatim}

\usepackage[nolist,nohyperlinks]{acronym}
\usepackage{xspace}
\usepackage{amsmath}
\usepackage{amssymb}
\usepackage{amsfonts}
\usepackage{amsthm}
\usepackage{bbm} 
\usepackage{xcolor}
\usepackage{algorithm2e}
\usepackage{hyperref}
\usepackage{subcaption} 
\usepackage{booktabs}
\usepackage[normalem]{ulem} 
\usepackage{tikz}

\newtheorem{theorem}{Theorem}

\begin{acronym}
\acro{CP}{{Conformal Prediction}}
\acro{ML}{{Machine Learning}}
\acro{AI}{{Artificial Intelligence}}
\acro{UQ}{{Uncertainty Quantification}}
\acro{SIS}{{Semantic Image Segmentation}}
\acro{CV}{{Computer Vision}}
\acro{MM}{{Mathematical Morphology}}
\acro{CRC}{{Conformalized Risk Control}}
\acro{i.i.d.}{{independent and identically distributed}}
\acro{LAC}{{Least Ambiguous Set-Valued Classifiers}}
\acro{APS}{{Adaptive Prediction Sets}}

\end{acronym}

\newcommand{\uq}{\ac{UQ}\xspace}

\newcommand{\cp}{\ac{CP}\xspace}
\newcommand{\crc}{\ac{CRC}\xspace}
\newcommand{\ml}{\ac{ML}\xspace}

\newcommand{\Yhat}{\widehat{Y}\xspace}

\newcommand{\C}{\mathcal{C}}

\newcommand{\lbd}{$\lambda$\xspace}
\newcommand{\Clb}{\mathcal{C}_{\lambda}}

\newcommand{\dt}{\delta}
\newcommand{\lb}{\lambda}
\newcommand{\lbhat}{\hat{\lambda}}

\newcommand{\Lb}{\Lambda}

\newcommand{\fhat}{\hat{f}}

\newcommand{\Xtest}{X_{\text{test}}}
\newcommand{\Ytest}{Y_{\text{test}}}

\newcommand{\f}{\hat{f}}

\newcommand{\XY}{(X_i,Y_i)}

\usetikzlibrary{patterns} 
\definecolor{se_col}{HTML}{8ec3f1}
\definecolor{falseNegColor}{RGB}{216,27,96}
\newcommand{\seCross}[1][0.25]{%
  \begin{tikzpicture}[scale=#1, baseline=0.5ex]
    \fill[se_col] (1,0) rectangle (2,3); 
    \fill[se_col] (0,1) rectangle (3,2); 
    
    \draw[line width=0.5pt, color=darkgray] (0,0) grid (3,3);
    
    \fill (1.5,1.5) circle (0.20);
  \end{tikzpicture}%
}
\newcommand{\seSquare}[1][0.25]{%
  \begin{tikzpicture}[scale=#1, baseline=0.5ex]
    \fill[se_col] (0,0) rectangle (3,3);
    \draw[line width=0.5pt, color=darkgray] (0,0) grid (3,3);
    \fill (1.5,1.5) circle (0.20);
  \end{tikzpicture}%
}

\definecolor{predColor}{RGB}{30, 136, 229}

\definecolor{coveredColor}{RGB}{254,97,0}
\newcommand{\coveredPixels}[1][0.5]{%
  \begin{tikzpicture}[scale=#1, baseline=0.4ex]
    \fill[coveredColor] (0,0) rectangle (2,2);
    \draw[line width=0.5pt, color=darkgray] (0,0) grid (2,2);
  \end{tikzpicture}%
}
\definecolor{truePosColor}{RGB}{120, 94, 240}
\newcommand{\truePos}[1][0.5]{%
  \begin{tikzpicture}[scale=#1, baseline=0.4ex]
    \fill[fill=truePosColor] (0,0) rectangle (2,2);
    \draw[line width=0.5pt, color=darkgray] (0,0) grid (2,2);
  \end{tikzpicture}%
}
\definecolor{marginColor}{RGB}{143, 195, 242}
\newcommand{\marginPixels}[1][0.5]{%
  \begin{tikzpicture}[scale=#1, baseline=0.4ex]
    \fill[fill=marginColor] (0,0) rectangle (2,2);
    \draw[line width=0.5pt, color=darkgray] (0,0) grid (2,2);
  \end{tikzpicture}%
}

\newcommand{\repo}{\url{https://github.com/deel-ai-papers/consema}\xspace}

\begin{document}
\title{Conformal Prediction for Image Segmentation Using Morphological Prediction Sets}
\date{}

\author{
    Luca Mossina \quad Corentin Friedrich
    \vspace{0.5em}
    \\
        IRT Saint Exupéry\\
        Toulouse, France\\
        {\small \texttt{firstname.lastname@irt-saintexupery.com}}
}

\maketitle        
\begin{abstract}
Image segmentation is a challenging task influenced by multiple sources of uncertainty, such as the data labeling process or the sampling of training data.
In this paper we focus on binary segmentation and address these challenges using conformal prediction, a family of model- and data-agnostic methods for uncertainty quantification that provide finite-sample theoretical guarantees and applicable to any pretrained predictor.
Our approach involves computing nonconformity scores, a type of prediction residual, on held-out calibration data not used during training.
We use dilation, one of the fundamental operations in mathematical morphology, to construct a margin added to the borders of predicted segmentation masks.
At inference, the predicted set formed by the mask and its margin 
contains the ground-truth mask with high probability, at a confidence level specified by the user.
The size of the margin serves as an indicator of predictive uncertainty for a given model and dataset.
We work in a regime of minimal information as we do not require any feedback from the predictor: 
only the predicted masks are needed for computing the prediction sets.
Hence, our method is applicable to any segmentation model, including those based on deep learning;
we evaluate our approach on several medical imaging applications.\footnote{
Our code is available at \repo.
} 

~
\end{abstract}

\section{Introduction}
\label{sec:intro}

\uq is essential for ensuring the reliability of \ml models in critical fields like medical imaging \cite{Lambert_2024_trustworthy}. 
In image segmentation, uncertainties can stem from various sources, including data labeling and sampling.
If such predictions are part of a complex system, such as an automated aid in medical diagnostics, one needs to rigorously quantify the prediction errors. 
We use \cp \cite{Vovk_2005_algorithmic,Gupta_2022_nested}, a framework that provides model- and data-agnostic methods for \uq with finite-sample theoretical guarantees, applicable to any pretrained predictor.
It constructs prediction sets that contain the truth at a confidence level defined by the user,
using held-out 
i.i.d.\footnote{
    \cp also applies to \textit{exchangeable} data, which is a less strict requirement.
} 
calibration data from the same distribution as production data.
We focus on binary segmentation, where each pixel is classified as either belonging to the object (e.g., a tumor) or the background.

\textbf{Our contribution}.
We propose a novel approach to \cp for image segmentation,
working with a minimal set of hypotheses: only the binary prediction masks ($\Yhat$, Fig.~\ref{fig:wbc_nc_score_pred}) are needed. 
Unlike existing methods, we do not require access to the predictor $\f(\cdot)$ nor its feedback (e.g., logits);
thus, our method is applicable to black-box predictors, e.g., embedded in third-party software and machines or derived from complex foundation models \cite{Ma_2024_Med_SAM}.

We build \cp sets as margins to be added on the contours of masks  (Fig.~\ref{fig:wbc_nc_score_dil}) using morphological dilation.
The size of these margins depends on the nonconformity scores (Eq.~\ref{eq:score-smallest-margin}) measured on held-out calibration data.
This method can be used to validate a model (knowing the typical error we incur into, on production data), but also to provide a set of pixels (the ``conformal margin'') that are likely to contain the part of ground truth we might have missed.
Although we focus on medical imaging, it is applicable to any use case and any segmentation model.

\begin{figure} [t]
    \centering
    \begin{subfigure}[b]{0.36\textwidth}
        \centering
        \includegraphics[width=0.5\textwidth]{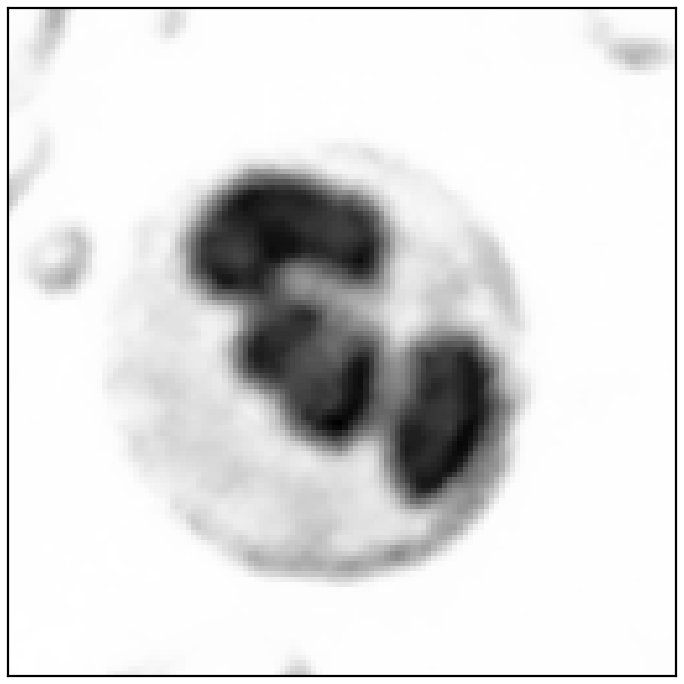}
        \caption{Input image $X$}
        \label{fig:wbc_nc_score_input_img}
    \end{subfigure}
    \begin{subfigure}[b]{0.36\textwidth}
        \centering
        \includegraphics[width=0.6\textwidth]{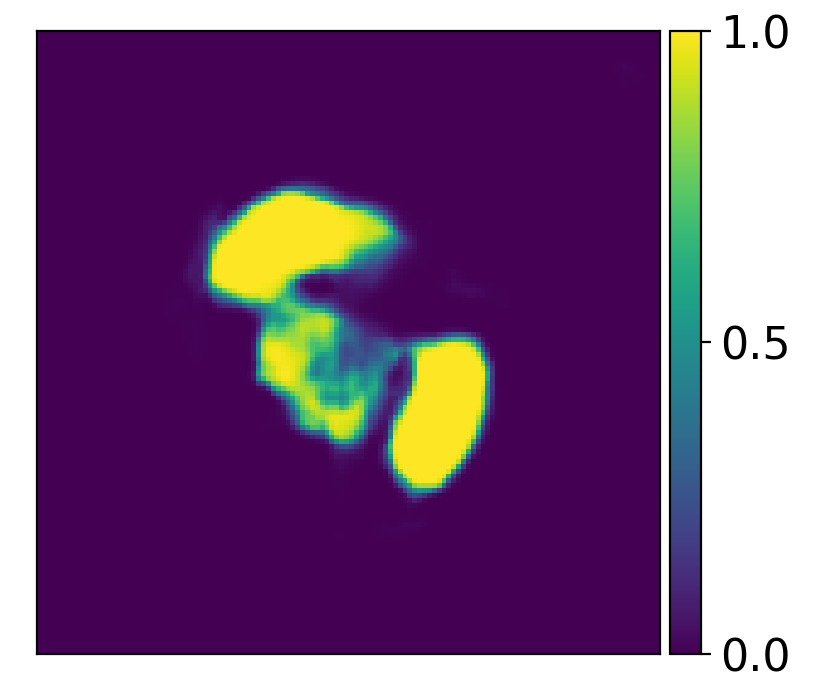}
        \caption{Sigmoid scores (unknown)}
        \label{fig:wbc_nc_score_soft_pred}
    \end{subfigure}

    \vspace{0.25cm} %
    
    \begin{subfigure}[b]{0.24\textwidth}
        \centering
        \includegraphics[width=0.94\textwidth]{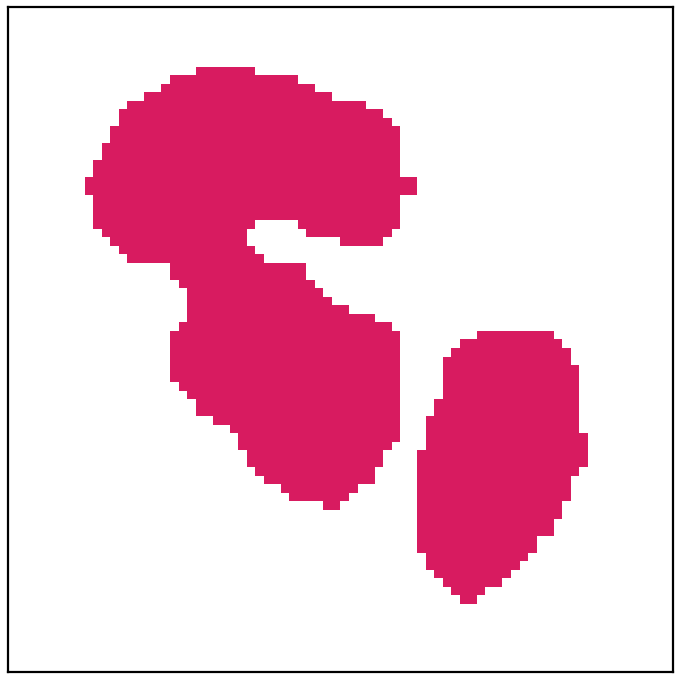}
        \caption{Target $Y$ \textcolor{white}{$\Yhat$}}
        \label{fig:wbc_nc_score_gt}
    \end{subfigure}
    \begin{subfigure}[b]{0.24\textwidth}
        \centering
        \includegraphics[width=0.94\textwidth]{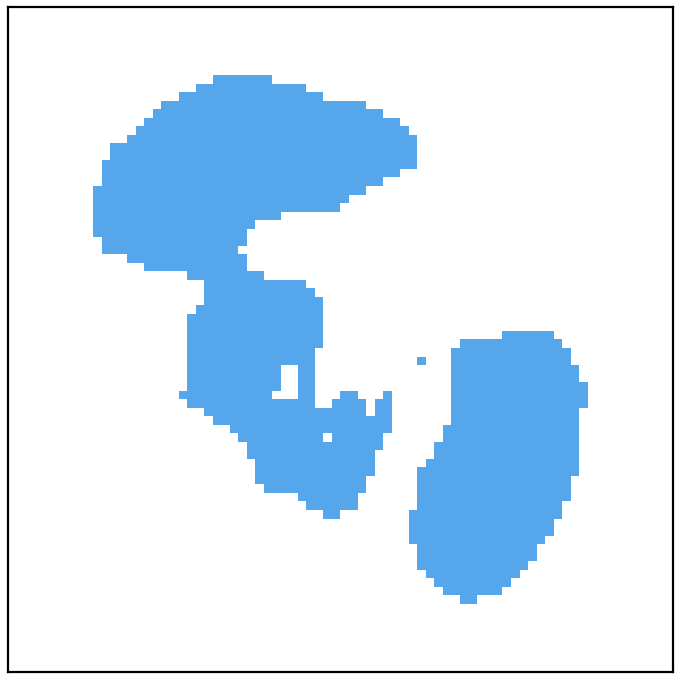}
        \caption{Prediction $\Yhat$}
        \label{fig:wbc_nc_score_pred}
    \end{subfigure}
    \begin{subfigure}[b]{0.24\textwidth}
        \centering
        \includegraphics[width=0.94\textwidth]{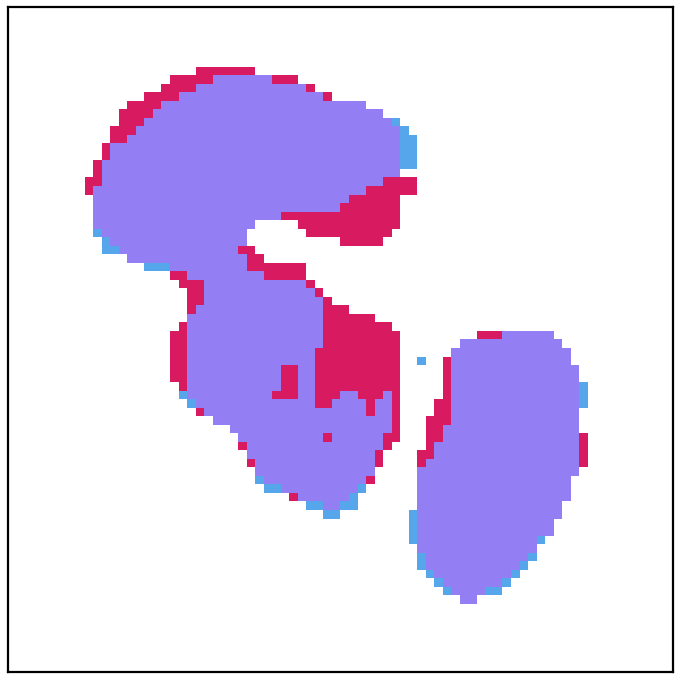}
        \caption{$Y \cap \Yhat$}
        \label{fig:wbc_nc_score_intersect}
    \end{subfigure}
    \begin{subfigure}[b]{0.24\textwidth}
        \centering
        \includegraphics[width=0.94\textwidth]{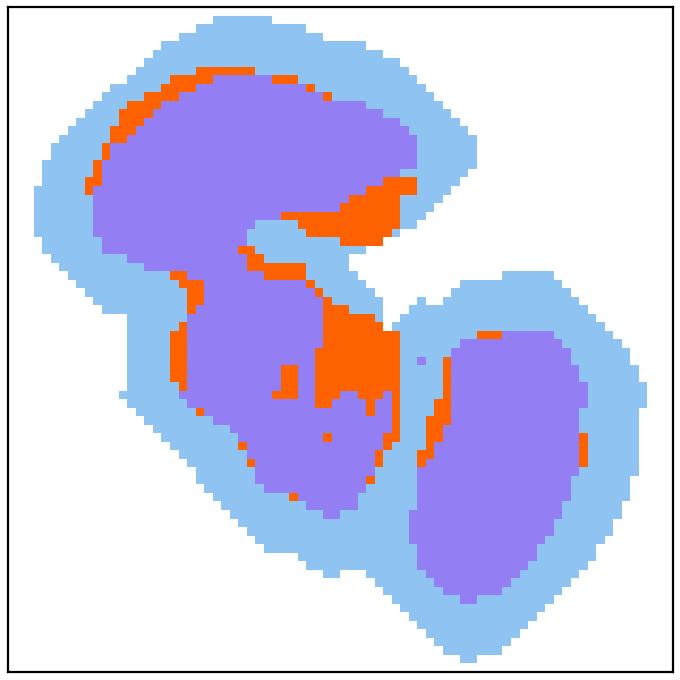}
        \caption{$\C_{6}(X) = \dt_{B}^{6}(\Yhat)$}
        \label{fig:wbc_nc_score_dil}
    \end{subfigure}

    \caption{Example: White Blood Cell (WBC) dataset \cite{Zheng_2018_WBC}, prediction (nucleus) with UniverSeg \cite{Butoi_2023_UniverSeg}. 
    \textbf{(a)} 
    Input image $X$.
    \textbf{(b)}
    Sigmoid scores $\fhat(X)$, assumed to be unavailable.
    \textbf{(c)}
    ground-truth mask $Y$.
    \textbf{(d)}
    predicted mask $\Yhat$.
    \textbf{(e)} 
    intersection of $Y$  
    and $\Yhat$ in purple (true positives).  
    \textbf{(f)}
    prediction set $\C_{\lb}(\Yhat)$: 
    adding a margin via $\lb = 6$ dilations of $\Yhat$ by structuring element $B=$ \seCross[0.12], 
    the missing pixels (e, in red) are covered,
    as per nonconformity score in Eq.~\eqref{eq:score-smallest-margin}.
    ~\\
    \textbf{Colors\,}
    \truePos[0.15]: true positives; 
    \marginPixels[0.15]: dilation margin; 
    \coveredPixels[0.15]: false negatives recovered. 
    }
    \label{fig:wbc_nc_score_intro_plot}
\end{figure}

\section{Background}
\label{sec:context}
\acf{CP} \cite{Vovk_2005_algorithmic,Angelopoulos_2024_theoretical_CP,DaVeiga_2024_tutorial_conformal}
constructs prediction sets $\C(X)$ that contain the ground truth $Y$ with probability $\mathbb{P}\{ Y \in \mathcal{C}(X) \} \geq 1 - \alpha$, where $\alpha \in (0,1)$ is a user-specified error level (also ``risk'').
We use inductive (or ``split'') \cp \cite{Papadopoulos_2002_inductive}, which computes nonconformity scores (prediction residuals) on held-out, labeled calibration data that are independent of the training data and follow the same distribution as the test data.  
The size of $\C(X)$ is often interpreted as a measure of uncertainty,\footnote{  
    \cp does not distinguish \cite{Mossina_2024_varisco} between \textit{aleatoric} and \textit{epistemic} uncertainty \cite{Huellermeier_2021_Aleatoric}.
}
as it depends on a quantile of the nonconformity scores.

\paragraph{Conformal Prediction in image segmentation.}
Using a threshold on the sigmoid scores,
\cite{Bates_2021_RCPS,Angelopoulos_2022_CRC,Blot_2024_automatically_adaptive_CRC}
construct prediction sets with distribution-free risk-controlling procedures in binary segmentation.
In \cite{Mossina_2024_varisco}, they extend the method of \cite{Angelopoulos_2022_CRC} to account for multiple classes at once, where each class channel can be seen as a binary mask.
Furthermore, \cite{Davenport_2024_conformal} builds inner and outer prediction sets that capture the ground truth with high probability, 
and they propose a nonconformity score based on the distance to the boundary of the masks.
The methods of \cite{Liu_2024_spatial_sacp,Bereska_2025_sacp} use a spatially-aware weighting of the scores, under the hypothesis of pixel-wise exchangeability.
Our work is also related to \cp for object detection \cite{deGrancey_2022_object,Li_2022_towards,andeol_2024_conformal,Andeol_2023_confident,Timans_2024_adaptive}, 
where a ``conformal'' margin is added around the bounding boxes.

\subsection{Nested prediction sets}
\label{sec:nested-sets}
Following the formulation of \cp based on {nested prediction sets} of \cite{Gupta_2022_nested}, 
let $X$ and $Y$ be the input features and target; 
let $\{ \C_{\lb}(X) \}_{\lb \in \Lb}$ be a sequence of prediction sets, where $\Lb$ is an ordered set (e.g., $\Lb \subset \mathbb{R}^+$). 
This is said to be a \textit{sequence of nested sets} when, 
for any $\lb \leq \lb'$, we have $\C_{\lb}(X) \subseteq \C_{\lb'}(X)$.
The nonconformity score induced by $\C_{\lb}(X)$ is then
$
    r(X,Y) = \text{inf} \, \big\{ \lb \in \Lb \,:\, Y \in \C_{\lb}(X) \big\}.
$
In words, the score is the smallest parameter $\lb$ such that the prediction set built from features $X$ contains the true $Y$ 
(see \cite{Gupta_2022_nested} for more examples).
Given calibration data $\XY_{i=1}^{n}$, one can compute the empirical quantile $\lbhat$ in \eqref{eq:nested-conformal-quantile} which is then used at inference to build $\C_{\lbhat}(X_{\text{test}})$.
\begin{equation}
\label{eq:nested-conformal-quantile}
    \hat{\lb} = \lceil(n+1)(1-\alpha)\rceil\text{-th largest score in } \big(r(X_i,Y_i)\big)_{i=1}^{n}.
\end{equation}

\section{Methods}
\label{sec:conosco-method}
Let $\XY_{i=1}^{n}$ be a sequence of calibration points where 
$X_i$ is the input image, and $Y_i$ its ground-truth mask labeled by an expert.
We notate with $\Yhat$ the predicted mask and we assume that the segmentation predictor $\f$ is unknown or inaccessible, based on deep learning or any other algorithm.
Note that $Y$ and $\Yhat$ are sets of points (i.e. pixels), hence the usual set notation applies.

For binary segmentation, 
we aim to avoid false negatives in the prediction:
that is, every ground-truth pixel is contained in the predicted mask,
denoted as $Y \subseteq \widehat{Y}$.
To do so,
the core of our proposal is to add a margin $\mu(\Yhat)$ around the prediction $\Yhat$ so that we cover all true positives: the prediction set is $\Clb(X) = \Yhat \cup \mu(\Yhat)$ and the condition becomes $Y \subseteq \Clb(X)$.\footnote{
    We write $\Clb(X)$ to be consistent with the literature, although we could write $\Clb(\Yhat)$, since we do not need access to $X$ nor the underlying predictor $\f$.
}
With respect to Fig.~\ref{fig:wbc_nc_score_intro_plot}.c, 
our method statistically covers the red pixels (false negatives), whereas the blue pixels are false positives and deemed innocuous.

\subsection{Nested sets via morphological dilations}
\label{sec:morpho-nested-sets}
We use \textit{morphological dilation}, the fundamental operator of mathematical morphology \cite{Matheron_1975_random_sets,Serra_1984_math_morpho_v1,Soille_2004_Morpholgical,Schmitt_2013_morpho_math,Gonzalez_2017_digital_image_proc,Blusseau_2023_morphologie}.
Binary dilation $\dt_{B}(\cdot)$ on a discrete set (e.g., a binary digital image) is performed using a structuring element (SE) $B$, which defines pixel connectivity.
Common choices include a $ 3 \times 3 $ cross (\seCross[0.12], 4-connectivity) and a $ 3 \times 3 $ square (\seSquare[0.12], 8-connectivity) \cite{Gonzalez_2017_digital_image_proc}.
One iteration of dilation passes $B$ over the image and assigns a value of 1 to all zero-valued pixels that have at least one neighboring 1-pixel under $B$. 

Our proposal for the prediction set is to choose a structuring element $B$ and apply a dilation $\lb \in \Lb \subseteq \mathbb{N}$ times:
\begin{equation}    
\C_{\lb}(X) 
    := \underbrace{(\delta_B \circ \delta_B \circ \dots \circ \delta_B)}_{\lb \text{ iterations}}(\Yhat) 
    = \delta_B^{\lb}(\Yhat),
\label{eq:set-iter-dilations}
\end{equation}
\noindent
with 
$
\dt^{0}(\Yhat) := \Yhat$ and $\dt_B^{\lb}(\Yhat) := \dt_{B}^{\,}(\dt^{\lb-1}_B(\Yhat))
$.
Also, the set $\mu^{\lb}(\Yhat) = \Clb(X) \setminus \Yhat $ is what we call the \textit{margin} of $\Clb(X)$.
Morphological dilation is extensive, so the dilated set always increases in size and contains the original set, until the whole image is covered.
It follows that for any (nonempty) prediction mask $\Yhat$, the sequence $(\dt_B^{\lb}(\Yhat))_{\lb \in \Lb}$ forms a sequence of nested sets.
Furthermore, it would be straightforward to extend this method to negative margins, using erosions on the background if the predictions were over-covering the ground truth.

Note that any operation that preserves the nested conditions of Sec.~\ref{sec:nested-sets} is applicable, such as chaining several structuring elements to induce specific shape on the margin or having SE's of variable size.
For example, it is possible to obtain the distance-based score of \cite{Davenport_2024_conformal} doing a single dilation with a structuring element $B(\lb)$ that grows with $\lb$, so that $B(\lb) \subset B(\lb')$ for any $\lb < \lb'$.
Then, the prediction set is
$
\C_{\lb}(X_i)  = \delta^{1}_{B(\lambda)} (\Yhat_i) 
$,
where, for $\lb = 0$, we have $B(0) := \varnothing$ and $\dt_{\varnothing}(\Yhat) := \Yhat$.
This can be a discrete approximation of a disc or any other shape. 
As above, it also holds that $(\dt^{1}_{B(\lb)}(\Yhat))_{\lb \in \Lb}$ is a sequence of nested sets. 

It is important that one rely exclusively on prior knowledge or training data when selecting the morphological operation and $B$. 
Using the calibration data for this purpose would violate the i.i.d. or exchangeability assumptions required by \cp.

\subsubsection{Nonconformity score.}
For a calibration pair $\XY$,
we define the score  as the smallest value $\lb \in \Lb$ such that at least $\tau \times 100\%$ of the ground-truth pixels in $Y_i$ are contained within the prediction set $\Clb(X_i)$, where the hyperparameter $\tau \in [0,1]$ is referred to as \textit{coverage ratio} \cite{Mossina_2024_varisco}.
More formally,
\begin{equation}
    r_{}\XY = 
    \text{inf}
    \left\{
        \lb \in \mathbb{N} \,:\, 
        \frac{|Y_i \cap \C_{\lb}(X_i)|}{|Y_i|} \geq \tau 
    \right\},
\label{eq:score-smallest-margin}
\end{equation}
\noindent where $| \cdot |$ is the number of elements (pixels) in a set. 
In some rare events, demanding to cover the entire truth can be overly conservative, that's why we introduce the hyperparameter $\tau$ that enables a trade-off.
For instance, a $\tau = 0.999$ implies that the user can accept up to $0.1\%$ of false negatives in $\Clb(X)$.

Since by construction we have that for any $\lb \leq \lb'$,
$\C_{\lb}(X_i) \subseteq \C_{\lb'}(X_i)$,
this can be applied into the formulation of conformal {nested prediction sets} \cite{Gupta_2022_nested,Angelopoulos_2022_CRC} (see Sec.~\ref{sec:nested-sets}) and the following holds:

\begin{theorem} 
Let $\lbhat$ be computed as in Eq.~\eqref{eq:nested-conformal-quantile}. 
Under the hypotheses of inductive conformal prediction \cite{Papadopoulos_2002_inductive,Gupta_2022_nested,Angelopoulos_2022_CRC}, 
for the nonconformity score in \eqref{eq:score-smallest-margin} induced by 
prediction sets \eqref{eq:set-iter-dilations}, it holds true that, for a new point $(\Xtest,\Ytest)$,
    \begin{equation}
    \label{eq:cp-proba-tau}
    \mathbb{P} \left[
        \frac{|\Ytest \cap \C_{\lbhat}(\Xtest)|}{|\Ytest|} \geq \tau 
        \right]
        \geq 1 - \alpha.
    \end{equation}
\end{theorem}
\begin{proof}
For any
$\tau \in [0,1]$, 
it suffices to set a binary loss to
$\ell(\Clb(X),Y) = \mathbbm{1}\{ \frac{|Y \cap \C_{\lb}(X)|}{|Y|} \not \geq \tau  \}$ 
(monotone in \lbd)
and apply \crc as per Theorem 1 in \cite{Angelopoulos_2022_CRC},
where they show that \crc with binary losses and \cp are statistically equivalent. 
$\square$
\end{proof}
The \cp guarantee in Eq.\eqref{eq:cp-proba-tau} is said to hold \textit{marginally}, i.e., on average over all possible inputs $\Xtest$ and on average over repeated draws of the calibration and test samples;
see \cite{Angelopoulos_2024_theoretical_CP} for statistical details.
Eq.~\eqref{eq:nested-conformal-quantile} implies that,
for an \textit{a priori} fixed $\alpha$, the sample size must be $n \geq \frac{1}{\alpha} - 1$.
Similarly, 
for a fixed calibration set of size $n$, the user can choose (prior to calibration) an error value $\alpha \geq \frac{1}{n+1}$.

\subsubsection{Conformalization algorithm}
\label{sec:cp-algo}
The ``conformalization'' of a (unknown) pretrained segmentation predictor $\fhat$ boils down to:
\begin{enumerate} \itemsep0em
    \item set $\alpha \in (0,1)$ and collect labeled calibration data $\XY_{i=1}^{n}$, with $n \geq \frac{1}{\alpha} - 1$;
    \item fix a coverage ratio $ \tau \in [0,1]$ and a $B$ for prediction set $\Clb(\cdot)$ in Eq.~\eqref{eq:set-iter-dilations};
    \item compute the nonconformity scores $\left(r\XY \right)_{i=1}^n$ as per Eq.~\eqref{eq:score-smallest-margin};
    \item compute the empirical quantile $\lbhat$ as in Eq.~\eqref{eq:nested-conformal-quantile};
    \item for a test prediction $\Yhat_{\text{test}}$, use $\lbhat$ in $\Clb(\cdot)$ and compute the dilated mask. 
\end{enumerate}

\section{Experiments}
\label{sec:experiments}
We ran our experiments with two segmentation models and three dataset groups.
First, we used the pretrained UniverSeg model \cite{Butoi_2023_UniverSeg}.\footnote{
\url{https://github.com/JJGO/UniverSeg}, accessed 2024-04-08.
}
We tested it on two public datasets as also evaluated in their paper: WBC (White Blood Cells) \cite{Zheng_2018_WBC}\footnote{
    Distributed as open source at: \url{https://github.com/zxaoyou/segmentation_WBC}. 
    We acknowledge the \textit{Jiangxi Tecom Science Corporation, China}, and the \textit{CellaVision blog} (\url{http://blog.cellavision.com/}) for providing the data.
}
and OASIS \cite{Marcus_2007_OASIS_data,Hoopes_2022_learning_oasis},\footnote{
    Source: \url{https://sites.wustl.edu/oasisbrains}, obtained via \url{https://github.com/JJGO/UniverSeg}. 
}
a neuroimaging dataset.
As in \cite{Angelopoulos_2022_CRC},
we also ran experiments using the PraNet \cite{Fan_2020_pranet} model\footnote{
    We reused the precomputed predictions and dataset as partitioned (training, test) by the authors of \cite{Angelopoulos_2022_CRC}. 
    See \url{https://github.com/aangelopoulos/conformal-prediction}.
}
and  the collection of datasets it was trained on, covering polyp segmentation in colonoscopy images:  
ETIS \cite{Silva_2014_ETIS_dataset},
CVC-ClinicDB \cite{Bernal_2015_CVC_Clinic_Dataset},
CVC-ColonDB \cite{Tajbakhsh_2015_CVC_ColonDB},
EndoScene \cite{Vazquez_2017_EndoScene},
and Kvasir \cite{Pogorelov_2017_kvasir_data}.

For conformalization, we randomly shuffles and partitioned the original test data into calibration and proper test, applied the algorithm in Sec.~\ref{sec:cp-algo} and compute the following metrics:
\begin{align}
&\text{Empirical coverage: }
& \text{Cov}(\lbhat; \C, \tau) 
    & =  \frac{1}{n_{\text{test}}} \sum_{i=1}^{n_{\text{test}}} \mathbbm{1} \left\{ 
        \frac{|Y_i \cap \C_{\lbhat}(X_i)|}{|Y_i|} \geq \tau 
\right\},
        \label{eq:emp-cov}
\\
&\text{Stretch \cite{andeol_2024_conformal}: } 
    & \phi(\lbhat; \C) 
    &= \frac{1}{n_{\text{test}}} \sum_{i=1}^{n_{\text{test}}}
        \frac{|\C_{\lbhat}(X_i)|}{|\Yhat_i|}.
\label{eq:stretch}
\end{align}
\noindent
Statistical coverage being a random quantity,
the {empirical coverage} is the evaluation of a realization of Eq.~\eqref{eq:cp-proba-tau}. 
The stretch $\phi$ tells, on average, how much larger the prediction sets are with respect to $\Yhat$ (lower is better).
We also report the average empirical quantile $\lbhat$, which indicates how many dilations were necessary to attain the specified coverage.
For all the metrics, we report the average and standard deviation over 36 runs (i.e., shuffling and partitioning the data).

\subsection{Results}
\label{sec:results}
The results in Table~\ref{tab:expe_results} show that, as expected, 
the empirical coverage is greater than the nominal value $1 - \alpha$ on average over multiple runs, that is, our \cp procedure constructs statistically valid prediction sets ($\text{Cov} \geq 1 - \alpha$).
As for the size of the prediction set, we see how the stretch and $\lbhat$ increase for higher $\alpha$ and $\tau$, to compensate for the stricter requirements imposed by the user.
In Table~\ref{tab:expe_results} we can also see how \cp can be used to evaluate a model: for a given risk $\alpha$ and coverage ratio $\tau$, the margin is small (e.g., $1.056$ for WBC and UniverSeg at $\alpha=0.1$ and $\tau=0.9$) when the underlying predictor is already satisfying (empirically) the statistical requirements.
\begin{table}
\setlength{\tabcolsep}{1pt}
    \begin{tabular}{cc c ll c llll rr}
         \multicolumn{1}{c}{Model}  
         & \multicolumn{1}{c}{Dataset}  
         & ~~ 
         & \multicolumn{1}{l}{$1 - \alpha$} ~~
         & \multicolumn{1}{c}{$\tau$} 
         & ~~ 
         & \multicolumn{2}{c}{Cov} 
         & \multicolumn{2}{c}{$\phi$}
         & \multicolumn{2}{c}{avg $\lbhat$} \\
         \midrule
PraNet & polyps & & $0.9$ & $0.9$  & & 0.909 & {\scriptsize (0.023)$^{\dagger}$ } & 1.253 & {\scriptsize (0.131)$^{\dagger}$} & 6.083 & {\scriptsize (2.980)$^{\dagger}$} \\
& & & $0.9$ & $0.99$  & & 0.898 & {\scriptsize (0.029) } & 1.780 & {\scriptsize (0.270)} & 17.500 & {\scriptsize (5.593)} \\
& & & $0.9$ & $0.999$  & & 0.904 & {\scriptsize (0.023) } & 2.031 & {\scriptsize (0.309)} & 22.500 & {\scriptsize (5.969)} \\
& & & $0.95$ & $0.9$  & & 0.953 & {\scriptsize (0.020) } & 3.093 & {\scriptsize (0.857)} & 41.389 & {\scriptsize (14.349)} \\
& & & $0.95$ & $0.99$  & & 0.955 & {\scriptsize (0.019) } & 4.144 & {\scriptsize (0.971)} & 58.611 & {\scriptsize (14.946)} \\
& & & $0.95$ & $0.999$  & & 0.962 & {\scriptsize (0.019) } & 4.840 & {\scriptsize (0.959)} & 69.167 & {\scriptsize (14.238)} \\
UniverSeg & WBC & & $0.9$ & $0.9$  & & 0.952 & {\scriptsize (0.044) } & 1.108 & {\scriptsize (0.065)} & 1.056 & {\scriptsize (0.630)} \\
& & & $0.9$ & $0.99$  & & 0.932 & {\scriptsize (0.055) } & 1.646 & {\scriptsize (0.261)} & 6.278 & {\scriptsize (2.514)} \\
& & & $0.9$ & $0.999$  & & 0.928 & {\scriptsize (0.061) } & 1.865 & {\scriptsize (0.284)} & 8.389 & {\scriptsize (2.686)} \\
& & & $0.95$ & $0.9$  & & 0.984 & {\scriptsize (0.033) } & 1.449 & {\scriptsize (0.365)} & 4.361 & {\scriptsize (3.523)} \\
& & & $0.95$ & $0.99$  & & 0.985 & {\scriptsize (0.024) } & 2.448 & {\scriptsize (0.664)} & 13.667 & {\scriptsize (5.826)} \\
& & & $0.95$ & $0.999$  & & 0.985 & {\scriptsize (0.025) } & 2.776 & {\scriptsize (0.735)} & 16.528 & {\scriptsize (6.236)} \\
UniverSeg & OASIS & & $0.9$ & $0.9$  & & 0.969 & {\scriptsize (0.040) } & 1.774 & {\scriptsize (0.061)} & 4.500 & {\scriptsize (0.507)} \\
& & & $0.9$ & $0.99$  & & 0.941 & {\scriptsize (0.049) } & 2.198 & {\scriptsize (0.034)} & 12.333 & {\scriptsize (0.756)} \\
& & & $0.9$ & $0.999$  & & 0.940 & {\scriptsize (0.038) } & 2.247 & {\scriptsize (0.030)} & 15.417 & {\scriptsize (0.996)} \\
& & & $0.95$ & $0.9$  & & 0.997 & {\scriptsize (0.007) } & 1.845 & {\scriptsize (0.040)} & 5.194 & {\scriptsize (0.401)} \\
& & & $0.95$ & $0.99$  & & 0.990 & {\scriptsize (0.014) } & 2.232 & {\scriptsize (0.033)} & 14.222 & {\scriptsize (1.124)} \\
& & & $0.95$ & $0.999$  & & 0.993 & {\scriptsize (0.014) } & 2.259 & {\scriptsize (0.028)} & 17.917 & {\scriptsize (0.937)} 
         \\
    \end{tabular}
    \caption{Experiments with structuring element $B$ = \seCross[0.12] averaged across 36 runs.
    $^{\dagger}$: standard deviation.
    }
    \label{tab:expe_results}
\end{table}

Finally, despite working in the restrictive setting without feedback from the predictor (no sigmoid scores), we show that there are cases where more information does not improve the prediction sets: 
for PraNet on the polyps dataset, we also compute conformal sets using a threshold on the sigmoid, so that as it is lowered, more pixels are included.
The average stretch across several configurations (Tab.~\ref{tab:dil_vs_threshold_pranet}) is considerably larger than with our method.
In this model and dataset configuration, our method performs better than approaches that use more information (e.g., sigmoid scores): not only it produces statistically valid sets, but it also avoids introducing artifacts in the conformal margin.
This phenomenon is visible in Fig.~\ref{fig:thresholding-artifacts}, where the margins are uninformative and seem to be due more to noisy scores in the underlying sigmoid rather than to actual uncertainty scoring (e.g., due to the type of training loss \cite{Azad_2023_loss_segmentation}).
On the other hand, in cases where distant zones are missed by the segmentation mask $\Yhat$, our method cannot recover those areas except by using a large margin.
In this case, using the sigmoid (if available), may be a viable option; 
in this a sense, our approach can be seen as complementary to other \cp methods.

\begin{table}
\setlength{\tabcolsep}{3pt}
    \centering
    \begin{tabular}{cc l l rr}
         \multicolumn{1}{c}{Model}  
         & \multicolumn{1}{c}{Dataset}  
         & \multicolumn{1}{l}{$1 - \alpha$} ~~
         & \multicolumn{1}{l}{$\tau$} 
         & \multicolumn{1}{c}{$\phi_{\text{morphology}}$} 
         & \multicolumn{1}{c}{$\phi_{\text{thresholding}}$}
         \\
         \midrule
        PraNet & polyps & $0.9$  & $0.9$   & 1.253  & \textbf{1.218}   \\
               &        & $0.9$  & $0.99$  & \textbf{1.780}  & 8.950    \\
               &        & $0.9$  & $0.999$ & \textbf{2.031}  & 15.001   \\
               &        & $0.95$ & $0.9$   &  3.093 & \textbf{3.010} \\
               &        & $0.95$ & $0.99$  &  \textbf{4.144} & 15.189 \\
               &        & $0.95$ & $0.999$ &  \textbf{4.840}  & 16.062 \\
    \end{tabular}

    \caption{Comparison with the first rows in Tab.~\ref{tab:expe_results}. Here, conformalization is done both by morphological dilation and thresholding of the sigmoid as in \cite{Angelopoulos_2022_CRC}. 
    }
    \label{tab:dil_vs_threshold_pranet}
\end{table}

\begin{figure}
\centering
    \begin{subfigure}[b]{0.20\textwidth}
        \centering
        \includegraphics[width=0.94\textwidth]{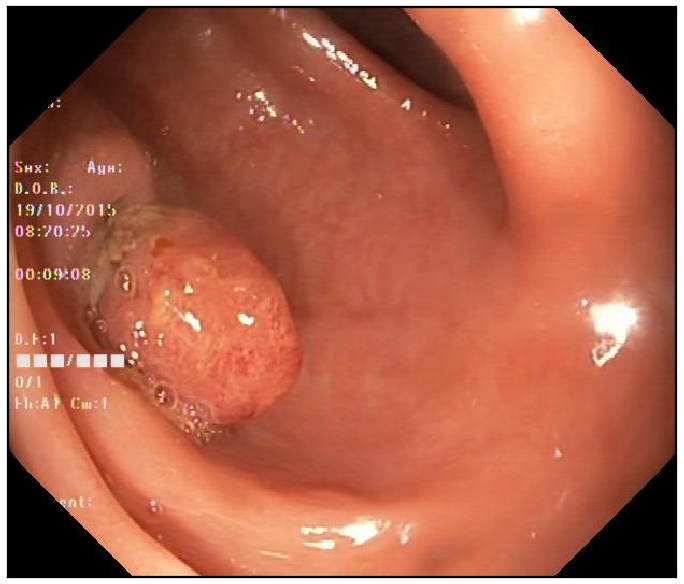}
        \caption{}
    \end{subfigure}
    \begin{subfigure}[b]{0.19\textwidth}
        \centering
        \includegraphics[width=0.94\textwidth]{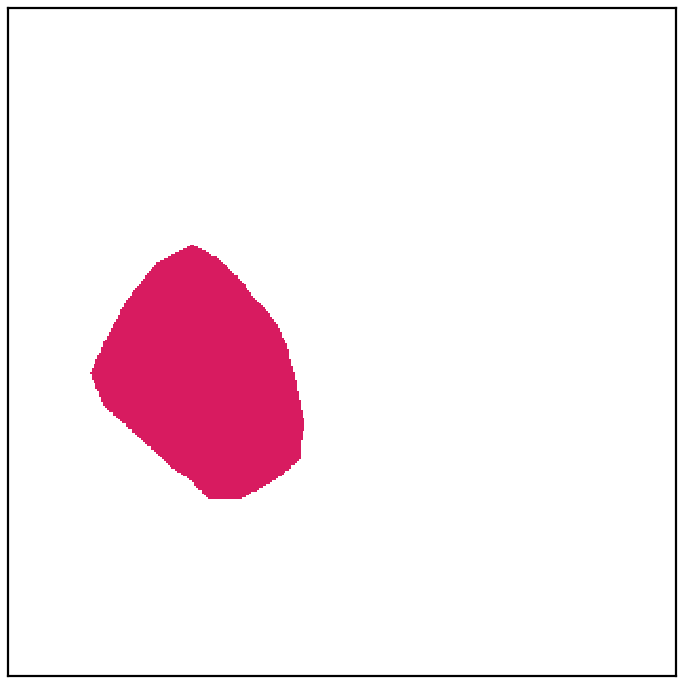}
        \caption{}
    \end{subfigure}
    \begin{subfigure}[b]{0.19\textwidth}
        \centering
        \includegraphics[width=0.94\textwidth]{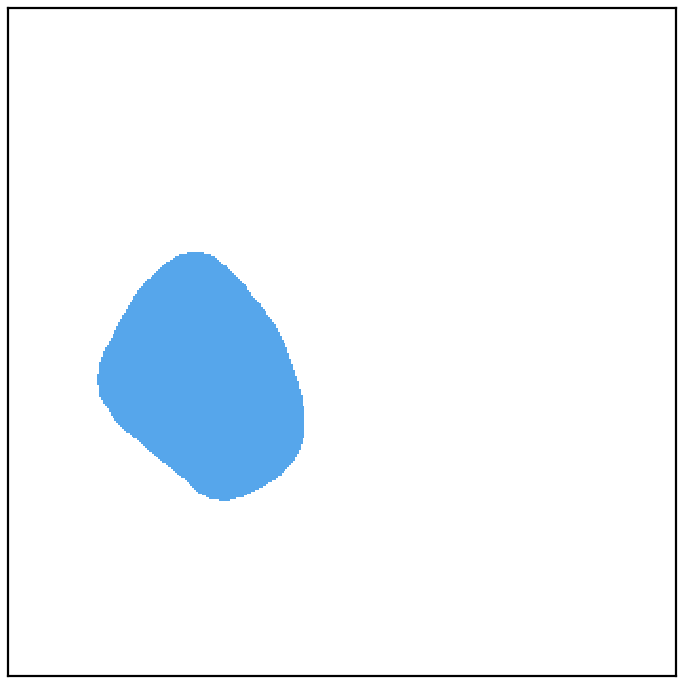}
        \caption{}
    \end{subfigure}
    \begin{subfigure}[b]{0.19\textwidth}
        \centering
        \includegraphics[width=0.94\textwidth]{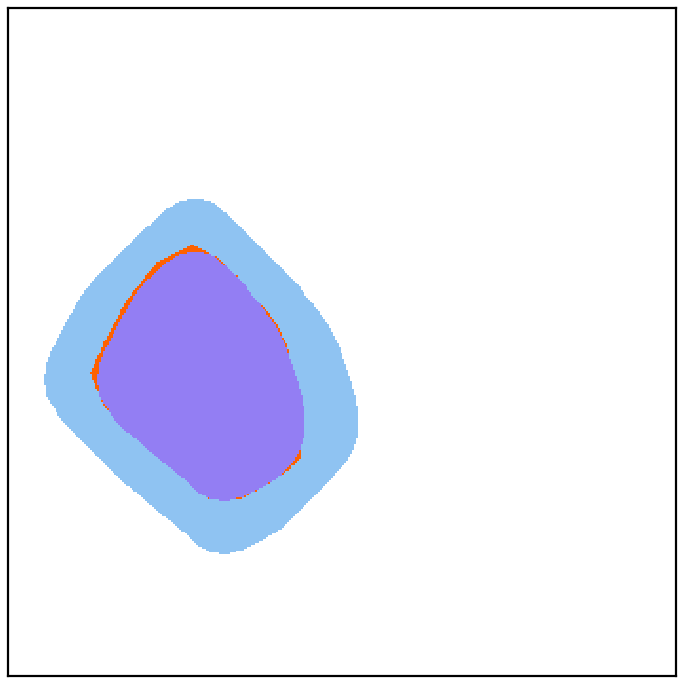}
        \caption{}
    \end{subfigure}
    \begin{subfigure}[b]{0.19\textwidth}
        \centering
        \includegraphics[width=0.94\textwidth]{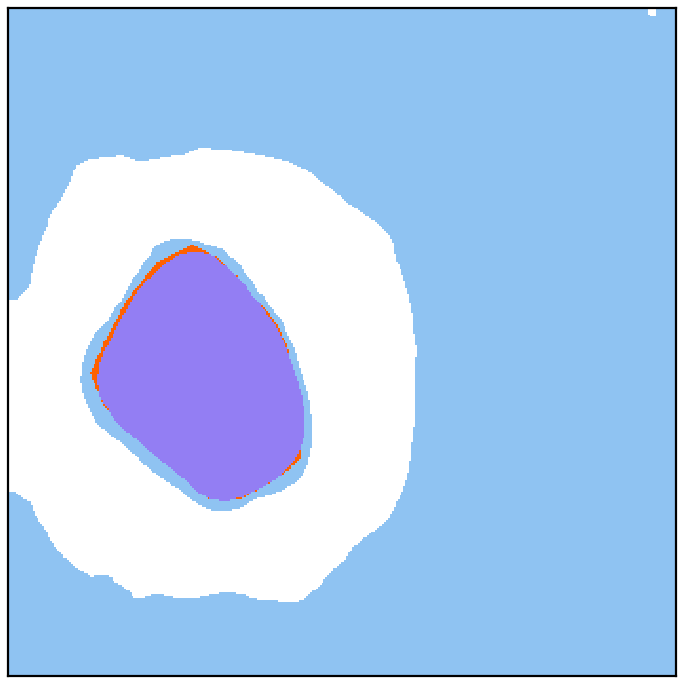}
        \caption{}
    \end{subfigure}
\caption{Example: polyps dataset (see Sec.~\ref{sec:experiments}), prediction with PraNet \cite{Fan_2020_pranet}.
For $\alpha=0.10$ and $\tau = 0.99$: 
\textbf{(a)} input image, 
\textbf{(b)} ground-truth mask, 
\textbf{(c)} predicted mask, 
\textbf{(d)} prediction set via dilation (Eq.~\ref{eq:set-iter-dilations}), 
\textbf{(e)} prediction set via thresholding on sigmoid (as in \cite{Angelopoulos_2022_CRC}).
Pixels in light blue (\marginPixels[0.15]) are the margin: 
in (d) it contains only pixels contiguous to the prediction (c)
while in (e), it does not necessarily do so because of the underlying sigmoid scores (not shown).
As shown in Tab.~\ref{tab:dil_vs_threshold_pranet}, for this model configuration the latter has much larger stretch (Eq.~\ref{eq:stretch}, lower is better).
White pixels represent the background.
}
\label{fig:thresholding-artifacts}
\end{figure}

\section{Conclusion}
In this paper, we proposed an approach that combines a fundamental operation in mathematical morphology, dilation, with Conformal Prediction to construct statistically valid prediction sets for image segmentation.
We achieved this in a restrictive framework with no internal knowledge of the model (e.g., sigmoid scores), where only the prediction masks are required.
We managed to preserve the original shape of the prediction and to remain robust to aberrant scores from such models.

Although we applied our algorithm to several benchmarks in medical imaging, our method can be used with any segmentation method that returns binary masks (e.g., standard thresholding/clustering approaches or more advanced ML models).
This includes \ml models that lack transparency (black boxes) and whose details are hidden from the end users, as well as algorithms that were not originally conceived for uncertainty quantification.

\textbf{Perspectives.}
A promising next step is to extend morphological sets to multiclass and instance segmentation, which are commonly used in the field of medical imaging.
In their basic form, \cp sets do not adapt to the input instance, and the theoretical guarantee holds on average: some images may be “harder” and require a larger margin, and vice versa.
This is an active field of research in \cp \cite{Gibbs_2023_conditional,Blot_2024_automatically_adaptive_CRC}, and the literature on Mathematical Morphology could provide new tools to build adaptive morphological prediction sets using training data.
Finally, we consider combining morphological sets with other approaches (e.g., thresholding) to leverage their respective strengths \cite{Teneggi_2023_k_RCPS}.

\section*{Acknowledgments}
The authors thank all the people and industrial partners involved in the DEEL project. This work has benefited from the support of the DEEL project,\footnote{\url{https://www.deel.ai/}} with fundings from the Agence Nationale de la Recherche, and which is part of the ANITI AI cluster.

\bibliographystyle{acm}
\bibliography{bibliography}

\end{document}